\documentclass[twoside,11pt]{article}
\usepackage{jair, theapa, rawfonts}
\usepackage{enumitem}
\usepackage[all]{xy}
\usepackage{amsmath} 
\usepackage{amsthm}
\usepackage{amsfonts}
\newtheorem{theorem}{Theorem}
\newtheorem{proposition}[theorem]{Proposition}%
\newtheorem{fact}[theorem]{Fact}
\newtheorem{corollary}[theorem]{Corollary}
\newtheorem{lemma}{Lemma}
\newtheorem{example}{Example}%
\newtheorem{remark}{Remark}%
\newtheorem{definition}{Definition}%

\jairheading{80}{2024}{1407-1436}{11/2023}{08/2024}
\ShortHeadings{Probabilities of the Third Type}
{Weitk\"amper}
\firstpageno{1407}

\begin{document}

\title{Probabilities of the Third Type: Statistical Relational Learning and Reasoning with Relative Frequencies}

\author{\name Felix Weitk\"amper \email felix.weitkaemper@lmu.de \\
       \addr Institut f\"ur Informatik, Ludwig-Maximilians-Universit\"at M\"unchen\\
       Oettingenstr. 67,
       80538 München, Germany }


\maketitle

\begin{abstract}
Dependencies on the relative frequency of a state in the domain are common when modelling probabilistic dependencies on relational data. For instance, the likelihood of a school closure during an epidemic might depend on the proportion of infected pupils exceeding a threshold. Often, rather than depending on discrete thresholds, dependencies are continuous: for instance, the likelihood of any one mosquito bite transmitting an illness depends on the proportion of carrier mosquitoes. Current approaches usually only consider probabilities over possible worlds rather than over domain elements themselves.  An exception are the recently introduced Lifted Bayesian Networks for Conditional Probability Logic, which express discrete dependencies on probabilistic data. We introduce functional lifted Bayesian networks, a formalism that explicitly incorporates continuous dependencies on relative frequencies into statistical relational artificial intelligence, and compare and contrast them with lifted Bayesian networks for conditional probability logic. Incorporating relative frequencies is not only beneficial to modelling; it also provides a more rigorous approach to learning problems where training and test or application domains have different sizes. To this end, we provide a representation of the asymptotic probability distributions induced by functional lifted Bayesian networks on domains of increasing sizes. Since that representation has well-understood scaling behaviour across domain sizes, it can be used to estimate parameters for a large domain consistently from randomly sampled subpopulations. Furthermore, we show that in parametric families of FLBN, convergence is uniform in the parameters, which ensures a meaningful dependence of the asymptotic probabilities on the parameters of the model. 

\end{abstract}

\section{Introduction}

Consider the different flavour of the following two statements:
 ``1\% of the population are suffering from the disease'', which is a statement about the relative frequency of an illness in the population; and
``Considering his symptoms, the likelihood that this patient is suffering from the disease is 20\%'', which is a statement about the degree of confirmation of the assertion that a particular patient has this illness, given the available evidence.

This distinction has first been methodically investigated in a strict logical framework
by \citeA{Carnap50}, who distinguished 
two concepts of probability, the ``degree of confirmation'', which
he calls ``probability\textsubscript{1}'', and the ``relative
frequency (in the long run)'', which he calls ``probability\textsubscript{2}''.
Carnap goes on to formalise probability\textsubscript{1} using a
probability measure defined over so-called \emph{state descriptions},
which we can identify as \emph{possible worlds} in more modern terminology.
Probability\textsubscript{2} on the other hand is interpreted by
the uniform measure on a given domain set itself.

Forty years later, in his seminal paper on the analysis of first-order
logics of probability, Joseph Halpern~\citeyear{Halpern90} divided approaches
to formalising probability in a relational setting along very similar
lines. Halpern refers to logics encoding relative frequencies
as Type I logics, while referring
to logics encoding degrees of belief as Type II logics\footnote{Somewhat unfortunately, this terminology is reverse to that of Carnap mentioned above.}. 
As a distinct category, Halpern also considers
logics that combine both by expressing a degree of belief in a statement that mentions relative frequencies.
He refers to those as Type III logics. Type III logics can express compound statements such as ``With a likelihood of at least 10\%, more than 60\% of the population will have been ill by the end of the year.'' 

In our contribution, we investigate the benefits of a Type III semantics
to the context of statistical relational artificial intelligence.
They include appropriateness to the intended
meaning of queries and a better grasp of extrapolation
behaviour, which in turn facilitates transfer learning across domain sizes.
We then propose functional lifted Bayesian networks (FLBN) as a Type III formalism and compare it with the formalism of lifted Bayesian networks with conditional probability logic (LBN-CPL)~\cite{Koponen20}, which we also identify as being of Type III.
After briefly summarising the learning algorithms that are available or adaptable to FLBN, we give a formal account of their extrapolation behaviour with increasing domain size, and explain how this addresses transfer learning between differently sized domains.
Indeed, we show that every FLBN is asymptotically equivalent to a particularly simple, quantifier-free lifted Bayesian network, which has well-understood benign scaling behaviour with respect to domains of increasing size.
Finally, we also show that for natural parametric families, convergence is uniform in the parameter space.
This makes our results applicable to situations such as parameter learning, where rather than a concrte single model, the parametric family itself is the object of study.    

\subsection{Queries Relating to Degrees of Belief vs Relative Frequency}

 Picking up the thread of disease modelling, we outline how probabilistic models and queries fit into the context of probability types. 

\begin{example}
For a domain of people and a graph of connections between them, consider
a model that represents how connected individuals infect each other.
A Type II query would ask ``what is the likelihood that a given individual
is infected at time $t$ (possibly given evidence)''.

This is clearly an interesting problem on this domain.
However, a main focus of epidemiologic modelling are adaptive interventions.
A trigger is set (such as ``1\% of the population
are infected'') and then some intervention is performed (such as
``schools are closed'') as soon as that trigger is reached.

Such trigger conditions that refer to relative frequencies
are very common.  (see
\cite[Table II]{BissetCDFMM14} for further examples). 
This naturally also leads to Type III queries, in which the likelihood
of a certain frequency event is addressed: ``How likely is it
that 1\% of the population will be infected within
the next four weeks?''
\end{example}

We discuss two different Type III formalisms, both of which incorporate Type I expressions into a Type II framework of lifted Bayesian networks. 
They are distinguished by the type of dependencies that they are designed to model. 
The framework of LBN-CPL~\cite{Koponen20} is built around the Type I language `Conditional Probability Logic', in which conditions such as ``At least 5\% of pupils in school $s$ are diagnosed'' can be expressed. Probabilities of relations or propositions at a child node would then depend on which condition is satisfied:
Often, however, dependencies are not discrete --- instead, they are of the form ``the higher (or lower) the proportion of $R$, the more likely is $Q$''. An example would be a mosquito-borne disease, in which the transmission risk of a single bite is directly proportional to the relative frequency of disease carriers in the mosquito population.
Therefore we introduce FLBN, where the probability of a child relation is a continuous function of the relative frequencies of parent nodes. 

\subsection{Asymptotics, Transfer Learning and Extrapolation }
We will see that systematically using Type I probabilities within an outer framework of lifted Bayesian networks also addresses a pertinent problem in parameter learning for statistical relational representations: discrepancies in domain size between training and test or deployment sets. Such a discrepancy could occur in different settings. On the one hand, it could be a deliberate choice since learning can be considerably more expensive than inference (which is known to be NP-hard in general~\cite{DalviS12}). Therefore, sampling small subsets of a complete dataset and then training the parameters on the sampled subsets could be much more efficient than learning the parameters on the entire set. This is recommended by the authors of the MLN system Tuffy~\cite{Tuffy11}, for instance. On the other hand, the size of the test set might be variable or unknown at training time. 

It is well-known, however, that in general the parameters that are optimal on a randomly sampled subset are not optimal on the larger set itself. 

\begin{example}\label{example:RLR}
Consider the typical example~\cite{PooleBKKN14} of a relational logistic regression with two unary relations $R$ and $Q$, and an underlying DAG $R(x) \longrightarrow Q(y)$.  For any $b \in D$, the probability of $Q(b)$ is given by 
$\mathrm{sigmoid} ( w * \lvert\{a \in D \mid R(a)\} \rvert )$,
where $w$ is a parameter to be learned from data. Now consider a training set of domain size 1000 in which 100 elements each satisfy $R$ and $Q$. Now assume that we sample a subset of 100 elements in which 10 elements each satisfy $R$ and $Q$. The optimal parameter on that subset would be a $w$ for which $\mathrm{sigmoid}(w * 10) = 10/100$, which turns out to be around -0.21. On the original set, however, the optimal parameter satisfies  $\mathrm{sigmoid}(w * 100) = 100/1000$, which  is around -0.021. Indeed, if we would transfer the original parameter to the larger set, it would predict a probability for $Q(y)$ of less than $10^{-9}$! 
\end{example}

\citeA{JaegerS18} showed that for certain projective families of distributions such a sampling approach provides a statistically consistent estimate of the optimal parameters on the larger set. However, projectivity is a very limiting condition; In fact, the projective fragments of common statistical relational frameworks that are isolated by \citeA{JaegerS18} are essentially propositional and cannot model any interaction between an individual and the population-at-large. For example, to make the relational logistic regression above projective,  the probability of $Q(a)$ must only depend on whether $R(a)$ is true, and not on any other elements of the domain. 
We will show that despite their larger expressivity Type III frameworks can be meaningfully approximated by projective distributions on large domains, allowing us to leverage the statistical consistency results for projective families of distributions.


We discuss two Type III formalisms for statistical relational AI. First we establish the Type I language of conditional probability logic (CPL), which is extended to a Type III framework called lifted Bayesian networks, following  \citeA{Koponen20}. We then propose a second, novel framework, able to express continuous rather than discrete dependencies on Type I probabilities. 

\subsection{Related Work}

As
\citeA{MuggletonC08} and \citeA{Schulte12} have noted, the
vast majority of statistical relational frameworks in use today are
of Halpern Type II --- they allocate a probability to each possible
world.

This includes Markov Logic Networks (MLN),
Probabilistic Logic Programming under the distribution semantics (PLP)
and approaches based on lifting Bayesian networks.
To give an idea of how these formalisms work, we provide with every formalism a small toy model of the following scenario:
There is a domain of people who can have a property of being ``social'' or not.
This is a unary relation, and should be positively correlated with how many ``friends'' (a binary relation) a person has. 

In their simplest form, MLNs~\cite{RichardsonD06} are given by a
set of first-order formulas $\varphi_{i}$ in a signature $\sigma$
annotated with real-valued weights $w_{i}$, as well as a domain $D$.
Then a probability measure over the set of all possible worlds $\mathfrak{X}$
on $D$ (i.\ e. $\sigma$-structures with domain $D$) is defined by
setting 
\begin{align}
\mathcal{P}(X=\mathfrak{X}):=\frac{1}{Z}\exp\left(\underset{i}{\sum}w_{i}n_{i}(\mathfrak{X})\right) \nonumber
\end{align}
where $n_i(\mathfrak{X})$ is the number of true groundings of $\varphi_i$ in $\mathfrak{X}$ and $Z$ is a normalisation constant ensuring that the
probabilities of all possible worlds sum to $1$. 
\begin{example}
  Consider the MLN with formulas $\varphi_1 := \mathrm{Social}(x) \land \mathrm{Friends}(x,y)$, $\varphi_2 := \mathrm{Social}(x)$,  $\varphi_3 := \mathrm{Friends}(x,y)$.
  They are associated with real-valued weights $w_1$, $w_2$ and $w_3$ respectively, where $w_1$ will be positive to affect a positive correlation between being social and having more friends.
  The other weights $w_2$ and $w_3$ can be set to any real number to calibrate the overall probability of socialness and friendship among domain individuals;
  if they are set to zero, then the positively weighted $\varphi_1$ will ensure that the probability of both friendship and socialness are above 50\%. 
\end{example}
  
The distribution semantics for PLP, first introduced in essence by \citeA{Poole93} and then named and explicitly proposed as a semantics for probabilistic logic programming by \citeA{Sato95} 
is given by a stratified logic program 
over independently distributed probabilistic ground facts on a domain $D$.
More precisely, let $\rho\subseteq\sigma$
be signatures and let $\alpha_{i}\in[0,1]$ for atoms $R_{i}(\vec{x})$
from $\rho$. Let $\Pi$ be a stratified Datalog program with
extensional signature $\rho$ and intensional signature $\sigma$.
Then every $\rho$-structure $\mathfrak{Y}$ on $D$ induces a $\sigma$-structure
$\mathfrak{X}$ on $D$ obtained by evaluating $\Pi$ with input $\mathfrak{Y}$.
In this way, $\Pi$ lifts the probability distribution on $\rho$-structures
given by making independent choices $R_{i}(\vec{x})$ of $\rho$-atoms
with probabilities $\alpha_{i}$, 
\begin{align}\nonumber
\mathcal{P}(X=\mathfrak{Y}):=
\left(\underset{a\in D,R_{i}\in\rho,\mathfrak{Y}\models R_{i}(\vec{a})}{\prod}\alpha_{i}\right)\left(\underset{a\in D,R_{i}\in\rho,\mathfrak{Y}\models\neg R_{i}(\vec{a})}{\prod}\left(1-\alpha_{i}\right)\right)
\end{align}
to a probability distribution on $\sigma$-structures with domain
$D$.
\begin{example}
  As is typical in probabilistic logic programming, modelling our example scenario requires us to expand the alphabet to include auxiliary terms that express the randomness inherent in the model. 
  At least two different probabilistic logic programs could be constructed for such a situation, depending on whether we model socialness to influence friendship or vice versa.
  In the latter case, the extensional signature is $\rho = \{\mathrm{Friends}, \mathrm{Makes\_social}\}$ and the intensional signature is given by $\rho \cup \{\mathrm{Social}\}$,
  and our logic program could have the form
  \[
  \mathrm{Social}(x) \leftarrow \mathrm{Friends}(x,y), \mathrm{Makes\_social}(y,x).
  \]
  along with probabilities $\alpha_1$ and $\alpha_2$ for the relation symbols $\mathrm{Friends}$ and $\mathrm{Makes\_social}$ respectively.
  Here, $\alpha_1$ sets the probability for any two individual to be social, and $\alpha_2$ models the probability of any individual friendship to cause a person to be considered social.
  These are considered to be independent causes for socialness, so if an individual has $n$ friends, then the probability for him to be social can be seen to be
  \[
  1 - \left(1 - \alpha_2\right)^n.
  \]
\end{example}
Note that by dividing the signature into an extensional and an intensional part, probabilistic logic programming cleanly separates the logical from the probabilistic part.
The probabilistic parts consists of mere probabilistic facts with no probabilistic relationships between them,
while the logical part consists of completely deterministic rules layered on top of those facts.
The expressivity needed for modelling real-world applications is recovered by introducing auxiliary relations as illustrated in the example above.  

Approaches based on lifting Bayesian networks, such as Relational
Bayesian Networks (RBN)~\cite{Jaeger02}, Bayesian Logic Programs
(BLP)~\cite{MuggletonC08} and Relational Logistic Regression (RLR)
\cite{KazemiBKNP14}, provide a template for a Bayesian network on
any given domain $D$, with a node for every possible ground atom $R(\vec{a})$,
$\vec{a}\in D$. The probability of every possible world is then defined
in the manner usual for Bayesian networks. 
We briefly sketch the RLR formalism as an example: 
Here, the nodes of the DAG are given by atoms of $\sigma$, and every node $Q(\vec{x})$ is annotated with a list  ${(\varphi_i, w_i)}_i$ of formulas whose relations are taken from the parents of $Q(\vec{a})$  and real-values weights $w_i$. The probability of $Q(\vec{a})$ given a grounding of its parents is as follows:
\begin{align}
\mathcal{P}(Q \left(\vec{a}) \right) := \mathrm{sigmoid} \left( \underset{i}{\sum} w_i n_i \right)  \nonumber
\end{align}
where $n_i$ is the number of true groundings of the formula $\varphi_i$.
\begin{example}
  The structure of the relational logistic regression network is subject to similar choices as the structure of the probabilistic logic program for the same scenario;
  we again choose to model a dependency of being considered social on having friends, rather than vice versa.
  This implies the underlying DAG structure to be $\mathrm{Friends}(x,y) \longrightarrow \mathrm{Social}(x)$,
  with a single atomic formula $\mathrm{Friends(x,y)}$ annotating the node $\mathrm{Social}(x)$ and an associated positive weight $w$.
  The probability of an individual being social is then given by $\mathrm{sigmoid}(nw)$,
  where $n$ is the number of friends that individual has.  
\end{example}

As a generalisation of stochastic grammars, Stochastic
Logic Programs (SLP) are very different
to the approaches above. Rather than providing a probability distribution
over possible worlds, they define a distribution over derivations
of a goal in a logic programming setting. Since deriving a goal
equates to querying a Herbrand base, this can be seen as defining
a distribution within that model. Therefore, SLPs can be classified
as Type I~\cite{MuggletonC08}.

A detailed exposition of all these formalisms, with a special focus on probabilistic logic programming, can be found in the recent textbook by \citeA{Riguzzi23}.

More explicitly of Type I is the class-based semantics for parametrised Bayesian networks suggested by \citeA{SchulteKKGZ14}. Syntactically, they are similar to the template Bayesian networks mentioned above, but probabilities are defined without grounding to any specified domain. Instead, they are interpreted as arising from a random choice of substitutions that instantiate the template nodes.

\citeA{Kern-IsbernerT10} discuss the integration of statistical probabilities into subjective beliefs in a more traditional context, in which different distributions satisfy a probabilistic knowledge base and then a maximum entropy criterion is used to specify a unique distribution.
However, \citeA{Jaeger02} provides for a `mean' aggregation function for
RBN and \citeA{Weitkaemper21} investigates scaled Domain-size
Aware RLR (DA-RLR) in which parameters are scaled with domain size.
Both formalisms induce a dependency on relative frequency (Type I
probability) of domain atoms, and we will see in Subsection \ref{ExpressFLBN}
below how those approaches are subsumed by our Type III framework.

\citeA{Koponen20} introduced {\em conditional probability logic (CPL)\/} as a language for expressing Type I probabilities, and defined {\em lifted Bayesian networks (LBN-CPL)\/} in which the conditional probability tables depend on such CPL formulas.
Therefore, LBN-CPL can clearly be identified as a Type III formalism.

However, unlike the statistical relational approaches mentioned above, LBN-CPL only encode discrete dependencies; rather than specifying an aggregation function, a node in a LBN-CPL is equipped with a finite set of probabilities triggered by CPL-formulas specifying relative frequency thresholds.
For instance, one could specify that the probability of a person contracting an illness depends on whether more than 1\% of the population are infected, but one cannot specify a continuous dependency on the percentage of people infected.

\section{Lifted Bayesian Networks and Conditional Probability Logic}\label{CPLLBN}

As CPL is
an extension of classical first-order logic, we will assume the syntax
and semantics of first-order logic as displayed in textbooks such as \citeA{EbbinghausFT94} to be known, and just briefly rehearse the terminology we employ here. 

\paragraph{Notation}We use the term \emph{relational signature} to mean a (possibly multi-sorted) relational signature (without equality). In other words, a relational signature is a set of relation symbols $R$, each of which is annotated with a finite list of sorts.
A (finite) \emph{domain} for such a relational signature consists of a finite set for every sort.
Then a \emph{sort-appropriate tuple} $\vec{a}$ for a relation symbol $R$ is a finite list of elements of the sets for those sorts given in the corresponding sort list of $R$.
Similarly, a sort-appropriate tuple for a finite list of sort variables is a list of elements of those sort domains given in the corresponding list of sort variables.
A \emph{$\sigma$-structure} on a domain is determined by a function that maps each relation symbol $R \in \sigma$ to a set of sort-appropriate tuples. When a first-order formula with free variables $\vec{x}$ is \emph{satisfied} by a sigma-structure with respect to a choice of sort-appropriate tuple for $\vec{x}$ is determined by the usual rules of first-order logic.
Given a $\sigma$-structure $\mathfrak{X}$ on a domain $D$, if $\varphi(\vec{x})$ is a first-order formula with free variables $\vec{x}$, we use $\lvert\lvert \varphi(\vec{x}) \rvert\rvert_{\vec{x}}$ to denote the fraction of sort-appropriate tuples from $D$ for which  $\varphi(\vec{x})$ holds. This is precisely the Type I probability or relative frequency of $\varphi(\vec{x})$ in $\mathfrak{X}$.

\subsection{Conditional Probability Logic }

CPL formulas over $\sigma$ are defined
inductively as follows, where the new
constructor $\left\Vert \varphi\mid\psi\right\Vert _{\vec{y}}$ should
be read as ``The conditional (Type-I-)probability that $\varphi(\vec{y})$
holds given that $\psi(\vec{y})$ is known to hold.''

:
\begin{definition}[\citeA{Koponen20}]\label{Def:CPLform}Let $\sigma$ be a (possibly multi-sorted) relational
signature (without equality). Then the set of \emph{conditional probability
formulas} over $\sigma$ is defined as follows:
\end{definition}
\begin{enumerate}
\item For every relation symbol $R$ in $\sigma$ of arity $n$ and appropriate
terms (i.e. variables or constants of the correct sorts) $t_{1},\ldots,t_{n}$,
$R\left(t_{1},\ldots,t_{n}\right)$ is a conditional probability formula.
\item If $\phi,\psi$ are conditional probability formulas, then $\neg\varphi,\varphi\wedge\psi,\varphi\vee\psi$
and $\varphi\rightarrow\psi$ are also conditional probability formulas.
\item If $\varphi$ is a conditional probability formula and $x$ is a variable,
then $\forall_{x}\varphi$ and $\exists_{x}\varphi$ are also conditional
probability formulas. 
\item For any $r\in\mathbb{R}_{+}$, conditional probability formulas $\varphi,\,\psi,\,\theta$
and $\tau$, and a tuple of distinct variables $\vec{y}$, the following
are conditional probability formulas:
\[
r+\left\Vert \varphi\mid\psi\right\Vert _{\vec{y}}\geq\left\Vert \theta\mid\tau\right\Vert _{\vec{y}}
\]
\[
\left\Vert \varphi\mid\psi\right\Vert _{\vec{y}}\geq\left\Vert \theta\mid\tau\right\Vert _{\vec{y}}+r
\]
\end{enumerate}
In ordinary first-order logic, a variable is \emph{bound} if it is
in the range of an existential or universal quantifier. In conditional
probability logic, a variable is also bound if it is in the range
of a construction of the form $\left\Vert \varphi\mid\psi\right\Vert _{\vec{y}}\geq\left\Vert \theta\mid\tau\right\Vert _{\vec{y}}+r$
or $\left\Vert \varphi\mid\psi\right\Vert _{\vec{y}}+r\geq\left\Vert \theta\mid\tau\right\Vert _{\vec{y}}$. 

The semantics for CPL are an extension of
the ordinary semantics of first-order logic.  \citeA{Koponen20} only defines the semantics on finite structures,
where the Type I probability measure on the domain is given by the
counting measure. This is also the setting most relevant to statistical
relational learning and reasoning.
\begin{definition}[\citeA{Koponen20}]
Let $\sigma$ be a relational signature, let $\mathfrak{X}$ be a
finite $\sigma$-structure (in the sense of ordinary first-order logic) on domain $D$. We define what it means for a conditional
probability formula $\varphi$ to \emph{hold in} $\mathfrak{X}$ with
respect to any sort-respecting interpretation of variables $\iota$,
written as $\mathfrak{X}\models_{\iota}\varphi$. Note that Clauses
1 through 3 of Definition \ref{Def:CPLform} are taken from ordinary
first-order logic. Therefore, we can copy the recursive definition
of atomic formulas and connectives directly from the corresponding
clauses for first-order logic. So assume that $\models_{\iota}$ has
been defined for $\varphi,\,\psi,\,\theta$ and $\tau$. Let $\iota_{\vec{y}:\vec{b}}$
be the variable interpretation obtained from $\iota$ by mapping $\vec{y}$
to $\vec{b}$ and otherwise following $\iota$. Furthermore, for any
term $t$ let $D_{t}$ be the domain of the sort of $t$. We then
define $\lvert\varphi\rvert_{\vec{y},\iota}$ to be the cardinality
of $\left\{ \vec{b}\in\underset{y\in\vec{y}}{\prod}D_{y}\mid\mathfrak{X}\models_{\iota_{\vec{y}:\vec{b}}}\varphi\right\} $. 

We set $\mathfrak{X}\models_{\iota}r+\left\Vert \varphi\mid\psi\right\Vert _{\vec{y}}\geq\left\Vert \theta\mid\tau\right\Vert _{\vec{y}}$
if and only if $\lvert\tau\rvert_{\vec{y},\iota}>0$, $\lvert\psi\rvert_{\vec{y},\iota}>0$
and $r+\frac{\lvert\varphi\wedge\psi\rvert_{\vec{y},\iota}}{\lvert\psi\rvert_{\vec{y},\iota}}\geq\frac{\lvert\theta\wedge\tau\rvert_{\vec{y},\iota}}{\lvert\tau\rvert_{\vec{y},\iota}}$. 

Analogously, $\mathfrak{X}\models_{\iota}\left\Vert \varphi\mid\psi\right\Vert _{\vec{y}}\geq\left\Vert \theta\mid\tau\right\Vert _{\vec{y}}+r$
if and only if $\lvert\tau\rvert_{\vec{y},\iota}>0$, $\lvert\psi\rvert_{\vec{y},\iota}>0$
and $\frac{\lvert\varphi\wedge\psi\rvert_{\vec{y},\iota}}{\lvert\psi\rvert_{\vec{y},\iota}}\geq\frac{\lvert\theta\wedge\tau\rvert_{\vec{y},\iota}}{\lvert\tau\rvert_{\vec{y},\iota}}+r$. 
\end{definition}
We introduce some intuitive shorthands: 
\begin{definition}
We write $\left\Vert \varphi\right\Vert _{\vec{y}}$ for $\left\Vert \varphi\mid\mathrm{true}\right\Vert _{\vec{y}}$,
expressing the unconditional probability of $\varphi(\vec{y})$, and
we write $r\geq\left\Vert \theta\mid\tau\right\Vert _{\vec{y}}$ and
$\left\Vert \varphi\mid\psi\right\Vert _{\vec{y}}\geq r$ for $r+\left\Vert \mathrm{false}\mid\mathrm{true}\right\Vert _{\vec{y}}\geq\left\Vert \theta\mid\tau\right\Vert _{\vec{y}}$
and $\left\Vert \varphi\mid\psi\right\Vert _{\vec{y}}\geq\left\Vert \mathrm{false}\mid\mathrm{true}\right\Vert _{\vec{y}}+r$
respectively.
We will refer to the set of all conditional probability formulas over a relational signature $\sigma$ as $\mathrm{CPL}(\sigma)$.
\end{definition}
CPL can express
the trigger functions in epidemic modelling that we had mentioned
in the previous section: for instance, ``at least 1\% of domain individuals
are infected'' is simply $\left\Vert \mathrm{Infected}(x)\right\Vert _{x}\geq0.01$.
Using conditional probabilities, more complex relationships can also
be expressed: ``at least 5\% of pupils at school $s$ are infected''
can be expressed as $\left\Vert \mathrm{Infected}(x)\mid\mathrm{Pupil}(x,s)\right\Vert _{x}\geq0.05$.
As an example utilising the full syntax including nested probability
quantifiers, we can even extend this to ``School $s$ is at least
at the median among schools in area $a$ by proportion of infected
pupils'': 
\begin{align}
0.5\geq&\Big\Vert \Vert \mathrm{Infected}(x)\Vert\mathrm{Pupil}(x,s)\Vert _{x} \nonumber \\
&\geq\Vert \mathrm{Infected}(x)\Vert\mathrm{Pupil}(x,y)\Vert _{x}\Big\Vert\mathrm{Located}(y,a)\Big\Vert _{y} \nonumber
\end{align}

While the particular syntax and the explicit treatment of \emph{conditional} probabilities is particular to CPL, it is similar in spirit to Halpern's~\citeyear{Halpern90} two-sorted Type I language and Keisler's~\citeyear{Keisler85} logic with probability quantifiers.  

\subsection{Lifted Bayesian Networks for Conditional Probability Logic}
To extend this Type I logic to a Type III representation and to integrate it with the independence assumptions from Bayesian networks, we follow  \citeA{Koponen20} in introducing \emph{Lifted Bayesian Networks}. 
\begin{definition}[\citeA{Koponen20}]
A \emph{Lifted Bayesian Network for Conditional Probability Logic (LBN-CPL)} over a relational signature $\sigma$  consists of the following:

\begin{enumerate}
\item An acyclic directed graph (DAG) $G$ with node set $\sigma$.

\item For each $R \in \sigma$ a finite tuple of formulas ${\big(\chi_{R, i}(\vec{x})\big)}_{i \leq \nu_R} \in \mathrm{CPL}(\mathrm{par}(R))$, where $\mathrm{par}(R)$ is the signature of the $G$-parents of $R$,  $\vec{x}$ is a sort-appropriate tuple of the correct length for $R$, such that
$\forall \vec{x} \big( \bigvee_{i = 1}^{\nu_R} \chi_{R, i}(\vec{x})\big)$ is valid (i.e. true in all $\mathrm{par}(R)$-structures) and if
$i \neq j$ then $\exists \vec{x} \big(\chi_{R, i}(\vec{x}) \wedge \chi_{R, j}(\vec{x})\big)$
is unsatisfiable. Such a tuple is called a \emph{partition}.

\item For each $R \in \sigma$ and each associated formula, 
a number denoted $\mu(R \ \mid \ \chi_{R, i})$ or $\mu(R(\vec{x}) \ \mid \ \chi_{R, i}(\vec{x}))$
in the interval $[0, 1]$.
\end{enumerate}

\end{definition}

The semantics of lifted Bayesian networks are defined by grounding to a Bayesian network with respect to a given domain.
We can view a $\sigma$-structure with a finite domain $D$ as a choice of truth value for each $R(\vec{a})$, where $R$ is a relation symbol in $\sigma$ and $\vec{a}$ is a tuple of elements of $D$ of the right length and the right sorts for $R$.
Therefore, defining a probability distribution over the space of possible $\sigma$-structures with domain $D$ is equivalent to defining a joint probability distribution over the $R(\vec{a})$, viewed as binary random variables.

\begin{definition}[\citeA{Koponen20}]
Consider an LBN-CPL $\mathfrak{G}$ and a finite domain $D$. 
Then the probability distribution induced by $\mathfrak{G}$ on the set of $\sigma$-structures with domain $D$ is given by the following Bayesian network: 
The nodes are given by $R(\vec{a})$, where $R$ is a relation symbol in $\sigma$ and $\vec{a}$ is a tuple of elements of $D$ of the right length and the right sorts for $R$.
There is an edge between two nodes  $R_1(\vec{a})$ and  $R_2(\vec{b})$ if there is an edge between $R_1$ and $R_2$ in the DAG $G$ underlying $\mathfrak{G}$.
It remains to define a probability table for every node $R(\vec{a})$:
Given a choice of values for $P(\vec{b})$ for all $P \in \mathrm{par}(R)$ and appropriate tuples $\vec{b}$ from $D$, the probability of $R(\vec{a})$ is set as $\mu(R \ \mid \ \chi_{R, i})$ for the unique $ \chi_{R, i}$ true for $\vec{a}$. This is well-defined as the truth of a CPL formula in a structure only depends on the interpretation of the relation symbols occurring in the formula, and since for every structure and every choice of $\vec{a}$ there is a unique $ \chi_{R, i}$ true for $\vec{a}$ by assumption.
\end{definition}

\section{Defining Functional Lifted Bayesian Networks}\label{defn}

While LBN-CPL enable us to express Type III conditions, they are intrinsically categorical: They do not allow the probability of  $R(\vec{a})$ to vary as a \emph{continuous} function of the Type-I-probabilities of first-order statements. Therefore we introduce FLBN, which are designed to do just that.

\begin{definition}
A \emph{functional lifted Bayesian network (FLBN)} over a relational signature $\sigma$ consists of the following:
\begin{enumerate}
\item  A DAG $G$ with node set $\sigma$.

\item For each $R \in \sigma$ a sort-appropriate tuple of variables $\vec{x}$ whose length is the arity of $R$ and a finite tuple  ${(\chi_{R, i}(\vec{x},\vec{y}))}_{i \leq n_R}$ of first-order $\mathrm{par}(R)$-formulas.

\item For each  $R \in \sigma$ a continuous function $f_R : {[0,1]}^{n_R} \rightarrow [0,1]$.
\end{enumerate}
 \end{definition}

The intuition behind FLBN is that for sort-appropriate ground terms $\vec{a}$, the probability of $R(\vec{a})$ is given by the value of $f$ applied to the tuple $(\left\Vert \chi_{R,i}\right\Vert _{\vec{y}})$.  Note that unlike in in LBN-CPL, the dependence on frequencies lies \emph{outside} the individual formulas, which are themselves purely first-order.

\begin{definition} 
Consider an FLBN $\mathfrak{G}$ and a finite domain $D$. 
Then the probability distribution induced by $\mathfrak{G}$ on the set of $\sigma$-structures with domain $D$ is given by the following Bayesian network: 
The nodes are given by $R(\vec{a})$, where $R$ is a relation symbol in $\sigma$ and $\vec{a}$ is a tuple of elements of $D$ of the right length and the right sorts for $R$.
There is an edge between two nodes  $R_1(\vec{b})$ and  $R_2(\vec{a})$ if there is an edge between $R_1$ and $R_2$ in the DAG $G$ underlying $\mathfrak{G}$.
It remains to define a conditional probability table for every node $R(\vec{a})$:
Given a choice of values for $P(\vec{b})$ for all $P \in \mathrm{par}(R)$ and appropriate tuples $\vec{b}$ from $D$, the probability of $R(\vec{a})$ is set as  $f_R({(\left\Vert \chi_{R,i}(\vec{a},\vec{y})\right\Vert _{\vec{y}})}_{i \leq n_R})$.
\end{definition}

We also draw attention to the degenerate case where the formulas $\chi_{R, i}$ have no free variables beyond $\vec{x}$.
Note that in this case the Type I probability can only take the values 1 and 0, depending on whether  $\chi_{R, i}$ is true or false. 
Thus, the only relevant function values of $f_R$ are those of tuples of 0 and 1.
In this way, it is possible to specify discrete dependencies on first-order formulas within the model.

Such a situation will always occur at the root nodes $R$ of the network. Since the parent signature is empty, ``true'' and ``false'' are the only possible formulas and therefore specifying the function $f_R$ for root nodes is equivalent to specifying a single probability $\mu(R)$. 

\begin{definition}
  A special case of degenerate FLBN are {\em quantifier-free lifted Bayesian networks}, which are those FLBN in which every formula $\chi_{R, i}$ is a quantifier-free formula all of whose variables are contained in the tuple $\vec{x}$.
\end{definition}
\begin{remark}\label{rem:qflbn}
Since in the grounding of a quantifier-free lifted Bayesian network, the entries of the conditional probability table for $R(\vec{a})$ only depend on $\chi_{R, i}(\vec{a})$, which is a Boolean combination of atomic formulas involving $\vec{a}$, the semantics given above can be simplified in this case by only including edges from $R_1(\vec{b})$ to  $R_1(\vec{a})$ if every entry in $\vec{b}$ also occurs in $\vec{a}$. 
\end{remark}

Under different names, quantifier-free lifted Bayesian networks are ubiquituous in statistical relational AI.
They correspond exactly to Jaeger's \citeyear{JaegerS18} relational Bayesian networks without aggregation functions or to determinate probabilistic logic programs.
In both cases, they are the only fragment of these languages known to be projective~\cite{JaegerS18}, and in the case of logic programs, the converse is also known to hold~\cite{Weitkaemper21a}.   
\section{Discussion of Type III Formalisms}
In this section we will discuss how LBN-CPL and FLBN can be used to express dependencies beyond the existing Type II formalisms, which learning algorithms are supported and how they enable transfer learning across domains of different sizes.
\subsection{Expressivity of Lifted Bayesian networks for Conditional Probability Logic}\label{ExpressLBN} 
Continuing the running example of infectious disease dynamics, CPL allows the expression of various trigger conditions. LBN-CPL allow us to express the actions taken when those conditions are met. Overall, one can model each of the policy decisions summarised by~\cite[Table II]{BissetCDFMM14} in LBN-CPL.
It seems at first that the requirement to list the probabilities for every case is very limiting.
For instance, it is not immediately clear how to express noisy-or, perhaps the most widespread aggregation function in current use. 
However, this apparent limitation can be overcome by introducing an auxiliary relation that captures the individual probability of every instance of the parent to influence the child node, a technique well-known in probabilistic logic programming under the distribution semantics to model probabilistic rules with probabilistic facts and determinate rules
\begin{example}
As an example, consider a model in which each of your infected contacts has an equal likelihood of 10\% to infect you with an illness. This can be expressed by the following network:
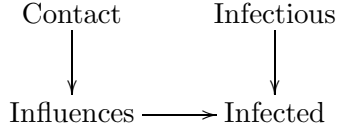
\[
\xymatrix{\mathrm{Contact}\ar[d] & \mathrm{Infectious}\ar[d]\\
\mathrm{Influences}\ar[r] & \mathrm{Infected}
}
\] 
Here, the auxiliary predicate "Influences" was introduced, which has a probability of 0.1  if and only if "Contact" was true and 0 otherwise. Then, "Infected" can simply be expressed as a deterministic dependency, with probability 1  if $\exists_q : (\mathrm{Influences}(x,y) \wedge \mathrm{Infected(y)}$ and probability 0 otherwise.
\end{example}

\begin{example}\label{exam:LBN-CPL}
 As an example combining several general features of statistical relational modelling with CPL, consider a policy where schools and workplaces are shut whenever there is either a diagnosed positive case in that school or workplace or when 0.1\% of the population are diagnosed with the disease. 
Additionally, there is a higher chance of contact between two people if they attend the same open school or workplace.
This can all be expressed in an LBN-CPL over a two-sorted signature as follows:
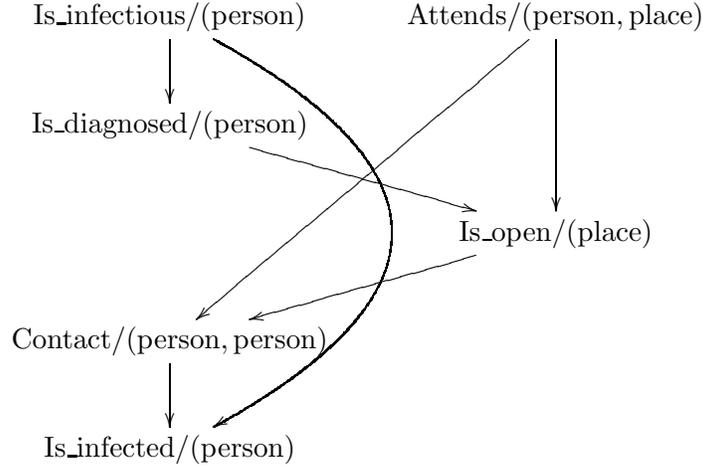
\[
\xymatrix{\mathrm{Is\_infectious/(person)}\ar[d]\ar@/_{7pc}/[dddd] & \mathrm{Attends/(person,place)}\ar[dddl]\ar[dd]\\
  \mathrm{Is\_diagnosed/(person)}\ar[dr]\\
 & \mathrm{Is\_open/(place)}\ar[dl]\\
  \mathrm{Contact/(person,person)}\ar[d]\\
\mathrm{Is\_infected/(person)}
}
\]

Assume that ``Attends'' is given by supplied data. Depending on whether this network is just one component in an iteration, where ``Is\_infectious'' depends on the ``Is\_infected'' of a previous time step, or stands alone, Is\_infectious might be given by data. Alternatively, it might be stochastically modelled, with a certain fixed probability. 
Then the conditional probabilities of the other four relations can be expressed as follows:
$\mathrm{Is\_diagnosed}(x)$ has a given probability  $p_1$ if $\mathrm{Is\_infectious(x)}$ is true, and a (lower) probability $p_2$ otherwise.

$\mathrm{Is\_open}(w)$ is deterministic, with probability 1 if 
\begin{align}
&\exists_x (\mathrm{Is\_diagnosed(x)} \wedge \mathrm{Attends}(x,w)) \nonumber \\
\vee &\left\Vert \mathrm{Is\_diagnosed}(x)\right\Vert _{x}\geq 0.01 \nonumber
\end{align}
 is true, and with probability 0 otherwise.

$\mathrm{Contact}(x,y)$ models contact close enough to transmit the disease. It could be set at a probability $p_3$ if 
\begin{align} 
\exists_w (\mathrm{Is\_open}(w) \wedge  \mathrm{Attends}(x,w) \wedge  \mathrm{Attends}(y,w)) \nonumber
\end{align}
and at a (lower) probability $p_4$ otherwise. The values of $p_3$ and $p_4$ will be varied depending on the transmissivity of the disease and the social structure of the population.  

Finally, $\mathrm{Is\_Infected}(x)$ can now be seen as a deterministic dependency, with probability 1 if  
\begin{align} 
\exists_y (\mathrm{Contact}(x,y) \wedge  \mathrm{Is\_Infected}(y)) \nonumber
\end{align}
 and probability 0 otherwise.
\end{example}

We conclude this subsection by showing that every LBN-CPL can be expressed by one containing only probabilistic facts and deterministic rules expressed in CPL. This is not known to hold for FLBN.
More precisely:

\begin{proposition}\label{distrosem}
  For every LBN-CPL $\mathfrak{G}$ over a signature $\sigma$ there is another LBN-CPL $\mathfrak{G}'$  in a signature $\sigma ' \supseteq \sigma$ such that $\mathfrak{G}$ and $\mathfrak{G}'$ induce the same probability distributions of $\sigma$-structures, and the following holds for $\mathfrak{G}'$:

  Every relation $R$ in $\sigma '$ is either a root node or a child of only root nodes, and in the latter case, all probabilities $\mu(R \ \mid \ \chi_{R, i})$ associated with $R$ are either 0 or 1.
\end{proposition}

\begin{proof}
The construction is reminiscent of the expression of Bayesian logic programs by the distribution semantics ~\cite[Section 4]{RiguzziS18}. 
For every relation symbol $R$ of $\sigma$ and every $\chi_{R, i}$, we add a new relation symbol $P_{R,i}$ to $\sigma$ of the same arity and sorts as $R$. This relation symbol is added as a root node and annotated with probability  $\mu(R \mid \chi_{R, i})$. We also add an edge from this new node to $R$.  Then we replace ``$\mu(R \mid \chi_{R, i})$'' with the number ``1'' and replace $\chi_{R, i}$ with $\chi_{R, i} \wedge P_{R,i}$.  We finally add a new formula $\chi_{R,\nu_{R}+1} := \underset{i}{\bigvee}(\neg P_{R,i})$ and set $\mu(R \mid \chi_{R,\nu_{R}+1})$ to 0. 
As we iterate this construction through all nodes,  we can successively replace any mention of a non-root relation symbol with its definition, which is given by $\underset{i}{\bigvee}  (\chi_{Q, i}\wedge P_{Q,i})$ and redraw the edges accordingly. Eventually, every relation symbol mentioned will be at the root of the DAG, as required. 
\end{proof} 
This is an interesting parallel with PLP and shows that while LBN-CPL are defined using a graphical model as a backbone, they could alternatively have been defined in the style of ``independent probabilistic facts and determinsitic rules.'' Note, however, that unlike in the case of PLP, the new relation $P(R_i)$ that we introduce has just the same arity and sorts as $R$, and is not in the scope of any quantifier. 

\subsection{Expressivity of Functional Lifted Bayesian Networks}\label{ExpressFLBN}
Functional lifted Bayesian networks can directly encode continuous dependencies on the relative frequency of features among a population.

\begin{example}
  Following on from the the theme of public health, assume now that a team of government health inspectors are tasked with evaluating whether an establishment satisfies public health regulations.
  They visit the establishment and make independent reviews, which are then taken into account by the local governing body.
  This situation could be modelled an FLBN in a two-sorted signature on the following DAG:
  \[
  \xymatrix{\mathrm{Meets\_standards/(place)}\ar[d] & \mathrm{Experienced/(person)}\ar[dl]\\
    \mathrm{Positive\_review/(person, place)}\ar[d]\\
    \mathrm{Passes/(place)}}
  \]
  Consider the options for the individual nodes in turn:

  In this model, Whether a place actually meets the standard or not is stochastic, in that a certain proportion of establishments meet the required standards and some don't.
  Similarly, there is a certain chance for any individual reviewer to be experienced.
  As these nodes do not have any parents, there are no further choices to be made beyond specifying their probabilities of occurrence.

  Whether a reviewer gives an establishment a positive review should depend most strongly on whether it actually deserves such a review, i.e. whether it meets the required standards.
  We can model this straightforwardly with the atomic formula $\mathrm{Meets\_standards(y)}$, which for any given place will give rise to a proportion of either 0 or 1.
  However, we also want to take into account that inexperienced reviewers are more likely to make mistakes and
  either give an undeserved positive review or withhold a positive review despite the standards being met.
  So, we use the second formula $\mathrm{Experienced(x)}$ and use both as input to our aggregation function.
  Although as written we demand continuity of our aggregation function, in this case, the input is a pair of Boolean values.
  Therefore, we can simply prescribe probabilities of writing a passing review for each of the 4 combinations of in/experienced reviewers and places meeting the standards or not.

  Considering the final node $\mathrm{Passes}$, we would like the probability of passing to depend on the proportion of positive reviews received.
  Depending on the real-life circumstances, we could apply a variety of continuous aggregation functions to this, but assume for now that we decide on the sigmoid function as an appropriate model.
  This would give rise to a single formula $\mathrm{Positive\_review}(x,y)$, aggregated over the person variable $x$, and an aggregation function
  $f:[0,1] \rightarrow [0,1]$, $f\left(\alpha\right) = \mathrm{sigmoid}\left(w(\alpha - p_o)\right)$,
  where a large positive $w$ would imply that it is almost guaranteed that a place passes muster if and only if the proportion of positive reviews it receives exceeds $p_0$.

  FLBNs also allow us to consider other factors in our decision.
  For instance, whether to give or not to give a pass to a place may depend not only on its own reviews in isolation,
  but on how those reviews compare to the average of reviews across establishments.
  To model this, take a second formula $\mathrm{Positive\_review}(x',y')$ which is aggregated over both variables.
  Its associated relative frequency models the overall average of all reviews across reviewers and places.
  The aggregation function would then be binary to accommodate this new information, and could for instance be given by
  $f(\alpha,\beta) = \mathrm{sigmoid}\left(w(\alpha - \beta + p_0)\right)$,
  which means (at positive values of $w$) that a place is likely to pass if its proportion of  positive reviews does not lie more than $p_0$ below the proportion of positive reviews across all establishments.

  Finally, one could instead consider that regardless of how large the domain of possible inspectors,
  having only positive reviews will always be significantly different than at least one inspector noting concerns.
  To model this, we simply add the additional formula $\forall_x\mathrm{Positive\_review(x,y)}$.
  This will always evaluate to 0 or 1, and can be added as an additional factor with an additional weight $w_2$,
  so that $f\left(\alpha,\beta\right) = \mathrm{sigmoid}\left(w_1(\alpha - p_o) + w_2\right)$.

  For any given pair of domains for persons and places, this FLBN will ground to a Bayesian network with one node for each ground atom,
  and an arrow between two nodes whenever there is an arrow between their relation symbols in the DAG above.
\end{example}

Continuous dependencies on relative frequencies are particularly important because they form the basis of the regression models from statistics.  From this point of view, one could see FLBN as a general framework for relational regression models contingent on Type I probabilities as data.

First we consider linear regression.
A linear regression model with Type-I-probabilities as data corresponds to the following families of functions:
\begin{align}\label{linear}
f (x_1, \ldots x_{n_R}) = \frac{w_1 x_1 + \cdots + w_n x_n}{n} + c      
\end{align}
where $w_1$ to $w_n$ and $c$ are coefficients that have to be chosen in such a way that the image of $[0,1]$ under $f$ is contained in $[0,1]$.
These families of functions suffice e.g. to express the ``arithmetic mean'' combination function in the relational Bayesian networks of \citeA{Jaeger97}.

As the target variable is binary, link functions are commonly used in so-called \emph{generalised linear models}~\cite{NelderW72} to transform the output of a linear regression into an element of the unit interval. 
Its functions are of the form 
\begin{align}\label{logistic}
f (x_1, \ldots x_n) = s(w_1 x_1 + \cdots + w_n x_n + c)       
\end{align}
where $s:\mathbb{R} \rightarrow [0,1]$ is a continuous function.
Common choices for $s$ are the sigmoid function  
\begin{align}
  s(x) = \frac{\exp(x)}{\exp(x) + 1} \nonumber
\end{align}
which leads to logistic regression, the cumulative density function of the normal distribution, which yields the probit model, and the asymmetric  function 
\begin{align}
    s(x) = 1-\exp(-\exp(x))  
\end{align}
which is related to the Poisson model.
The logistic regression model in particular specialises to the DA-RLR of \citeA{Weitkaemper21}. 

By recovering these existing frameworks and combination functions as special cases of FLBN, they are in scope of the rigorous analysis of asymptotic behaviour in Subsection \ref{asymptotics}.

Beyond those formalisms, functional lifted Bayesian networks can model a variety of other aggregation functions.
Consider for instance a unary random variable $R$ that depends on how far the proportion of $Q$s is from an optimum value $p$.
This can be modelled by a function
\begin{align}
  f(x) = \alpha e^{-\beta (x-p)^2}
\end{align}

where $\alpha < 1$ is the probability of $R$ when the proportion of $Q$ is optimal and $\beta$ determines how quickly the probability of $R$ drops as the proportion of $Q$ deviates from $p$. 

\section{Asymptotic Analysis of the Extrapolation Behaviour}\label{asymptotics}
We present a full analysis of the asymptotic behaviour of FLBN.\@
The setting is as follows:
On every finite domain $D$, a functional lifted Bayesian network  $\mathfrak{G}$  over a signature $\sigma$ induces a probability distribution $\mathbb{P}_{\mathfrak{G},D}$ on the set $\Omega_D$ of  $\sigma$-structures with domain $D$. The first thing to note is that the names of the elements of $D$ do not matter; all relevant information lies in the cardinalities of the sorts of $D$. Therefore we will assume from now on that our domain sorts consist of initial segments of the natural numbers, and we write $\mathbb{P}_{\mathfrak{G},\vec{n}}$ for the probability distribution on the sorts with $\vec{n}$ elements.
In an asymptotic analysis, we are interested in the limit of these probability distributions as the domain sizes of $D$ tend to infinity. If $\sigma$ is multi-sorted, we require the size of every sort to tend to infinity. More precisely, a sequence $(D_i)_n, i \in I, n \in \mathbb{N}$ of sorts tends \emph{uniformly to infinity} if there are $\delta_i \in \mathbb{R}$ for all $i \in I$ such that $\lvert {D_i}_n\rvert = \delta_in$ for all $n \in \mathbb{N}$. A technical difficulty here is that strictly speaking, the probability distributions are defined on different sets $\Omega_{\vec{n}}$, so it is unclear on which measure space a limit would even be defined. To be precise, we consider the measure space $\Omega_\infty$ given by all $\sigma$-structures with domain sorts $\mathbb{N}$. It is endowed with the $\sigma$-algebra generated by \emph{generating sets} of the following form: 
``Those $\sigma$-structures $\mathfrak{X}$ such that the $\sigma$-substructure of $\mathfrak{X}$ with domain $a_1 , \ldots a_m$ is given by $\mathfrak{Y}$'' for a tuple of domain elements  $a_1 , \ldots a_m$ and a $\sigma$-structure $\mathfrak{Y}$. 
This suffices to give a probability to any query about finite domains that we might possibly be interested in. 
On such a generating set, all but finitely many $\mathbb{P}_{\mathfrak{G},\vec{n}}$ are defined; indeed, $\mathbb{P}_{\mathfrak{G},\vec{n}}$ gives a probability to any structure with domain  $a_1 , \ldots , a_m$ as long as every $a_i$ is bounded by the entry of $\vec{n}$ corresponding to its sort. Furthermore, the probability of the generating sets completely determine the probability distribution on the measure space itself.
More formally, we can proceed as follows:
\begin{proposition}\label{prop:mass}
  Let $\Omega_\infty$ be the set of all $\sigma$-structures with domain sorts $\mathbb{N}$, and let $\mathfrak{H}$ be the system of subsets of $\Omega_\infty$ consisting of finite disjoint unions of sets of the form
  ``Those $\sigma$-structures $\mathfrak{X}$ such that the $\sigma$-substructure of $\mathfrak{X}$ with domain $a_1 , \ldots a_m$ is given by $\mathfrak{Y}$'' for a tuple of domain elements  $a_1 , \ldots a_m$ and a $\sigma$-structure $\mathfrak{Y}$ (\emph{generating sets}.
  Let $\mathbb{P}$ be a premeasure on $\mathfrak{H}$. Then $\mathbb{P}$ extends uniquely to a measure on the $\sigma$-algebra generated by $\mathfrak{H}$, and it is uniquely determined already by its values on the generating sets.
\end{proposition}
\begin{proof}
  One can easily check that $\mathfrak{H}$ defines a \emph{semiring of sets}, and thus the existence statement follows from Caratheodory's Extension Theorem~\cite[13.4.13]{Deitmar21}. As $\Omega_\infty$ can be written as a finite union of sets from $\mathfrak{H}$, this extension is also unique~\cite[16.1.5]{Deitmar21}. 
\end{proof}

We can make the following definitions:
\begin{definition}
 Two formulas of CPL with the same free variables $\vec{x}$, $\varphi$ and $\psi$, are \emph{asymptotically equivalent} over a lifted Bayesian network $\mathfrak{G}$  if  for any sequence $D_k$ of domains which tend uniformly to infinity, $\underset{k \rightarrow \infty}{\lim}\mathbb{P}_{\mathfrak{G},D_k}\left(\forall_{\vec{x}}\left( \varphi(\vec{x}) \leftrightarrow \psi(\vec{x})\right)\right) = 1 $

A probability distribution $\mathbb{P}_\infty$ on $\Omega_\infty$ is the \emph{asymptotic limit} of a lifted Bayesian network $\mathfrak{G}$ if for any sequence $D_k$ of domains  which tend uniformly to infinity and any generating set $A$ the limit of  $\mathbb{P}_{\mathfrak{G},\vec{n}}(A)$ equals $\mathbb{P}_\infty(A)$.
Two functional lifted Bayesian networks $\mathfrak{G}$ and $\mathfrak{G}'$ are \emph{asymptotically equivalent} if they share the same asymptotic limit $\mathbb{P}_\infty$. 
\end{definition}  

The discussion above does not imply, of course, that a given functional lifted Bayesian network actually has an asymptotic limit in that sense. 
However, there is a class of lifted Bayesian networks where the asymptotic limit is clear: those that define \emph{projective families of distributions}. 
\begin{definition}[\citeA{JaegerS18,JaegerS20}]
A family of probability distributions $(\mathbb{P}_{\vec{n}})$ on $\Omega_{\vec{n}}$ is \emph{projective} if for every generating set $A$ of $\Omega_\infty$ the sequence  $(\mathbb{P}_{\vec{n}}(A))$ is constant whenever it is defined. 
\end{definition}
Since clearly every constant sequence converges, every lifted Bayesian network inducing a projective family of distributions has an asymptotic limit by Proposition \ref{prop:mass}. Furthermore, if two families of distributions are asymptotically equivalent and both projective, they must in fact be equal. 
This leads us to the following observation:
\begin{proposition}\label{qFLBN are projective}
If $\mathfrak{G}$ is a \emph{quantifier-free lifted Bayesian network}, then $\mathfrak{G}$ induces a projective family of distributions. In particular, $\mathfrak{G}$ has an asymptotic limit.
\end{proposition}
\begin{proof} 
  Let $A\in\Omega_\infty$ be a generating set and let $\mathfrak{Y}$ be the corresponding structure on $a_1 , \ldots , a_m$ and let $(\mathbb{P}_{\vec{n}})$ be the sequence of probability distributions induced by  $\mathfrak{G}$.
  Then we need to show that for any ${\vec{n}}$ for which  $\mathbb{P}_{\vec{n}}(A)$ is defined this value is the same.
  So for any such ${\vec{n}}$ consider the Bayesian network $G_{\vec{n}}$ induced by $\mathfrak{G}$ on the domain ${\vec{n}}$.
  By Remark \ref{rem:qflbn}, all incoming edges into $R(\vec{a})$ come from a source node $R'(\vec{b})$ where all entries in $\vec{b}$ also occur in $\vec{a}$.
  Therefore, the calculation of the unconditional probability of the event $A$ does not involve any node with entries  outside of $a_1 , \ldots , a_m$.
  So let $G_{A,\vec{n}}$ be the complete Bayesian subnetwork on the nodes $R(\vec{a})$ for an $\vec{a} \in a_1 , \ldots , a_m$.
  By the discussion above, $G_{A,\vec{n}}$ suffices to calculate the probability of $A$.
  $G_{A,\vec{n}}$ has the same underlying DAG for any $\vec{n}$.
  As the conditional probability tables themselves are also identical for any $\vec{n}$, so is  $G_{A,\vec{n}}$  and thus the probability of $A$.   
\end{proof}
With these preparations out of the way, we can formulate our main result on the asymptotic behaviour of FLBN. 
\begin{theorem}\label{functionalLBN}
Let $\mathfrak{G}$ be an FLBN such that for all $n$-ary aggregation functions $f_R$, $f_R^{-1}\{0,1\} \subseteq {\{0,1\}}^n$. Then $\mathfrak{G}$ is asymptotically equivalent to a quantifier-free lifted Bayesian network $\mathfrak{G}'$. 
\end{theorem}
The condition that $f_R^{-1}\{0,1\} \subseteq {\{0,1\}}^n$ signifies that as long as among all of the parent formulas, there are both examples and counterexamples,
the relation at the child node is not deterministic.
This is important for asymptotic analysis since the behaviour of relations that occur with a very small but positive probability is radically different from that of relations whose probability is precisely 0.
For the former, as long as the domain size is big enough, there will almost always be at least a few examples of this relation present among domain individuals;
in the latter case, the relation will not occur, no matter how large the domain. 

We briefly turn to Koponen's analysis of LBN-CPL~\cite{Koponen20}. Koponen introduces the notions of a \emph{critical} number and a \emph{critical} formula. We will not rehearse their involved technical definitions here; it suffices for us to know that the definitions depend (only) on $\mathfrak{G}$, that every first-order formula is non-critical and that for all conditional probability formulas $\varphi, \psi, \theta, \tau$ and tuples of variables $\vec{y}$ for all but finitely many r the formulas $r+\left\Vert \varphi\mid\psi\right\Vert _{\vec{y}}\geq\left\Vert \theta\mid\tau\right\Vert _{\vec{y}}$ and 
$\left\Vert \varphi\mid\psi\right\Vert _{\vec{y}}\geq\left\Vert \theta\mid\tau\right\Vert _{\vec{y}}+r$ are non-critical.

 Note that the asymptotic characterisation of FLBN does not depend on non-criticality or any similar assumption in its statement. 
The main results of \citeA{Koponen20} are then (Theorems 3.14 and 3.16 respectively):
\begin{theorem}\label{AQE_CPL}
Over any given  LBN-CPL $\mathfrak{G}$, every non-critical conditional probability formula $\varphi$ is asymptotically equivalent to a quantifier-free formula $\psi$. 
\end{theorem}
As the proof is rather involved, we refer readers to \citeA{Koponen20}.

On the level of  networks, Koponen obtains a similar asymptotic convergence result:
\begin{theorem}\label{AQE_LBN}
Any  LBN-CPL $\mathfrak{G}$ all of whose partition formulas $\chi_{R,i}$ are non-critical is asymptotically equivalent to a quantifier-free lifted Bayesian network.
\end{theorem}
While \citeA{Koponen20} gives a stand-alone proof, we can derive it straightforwardly from Theorem \ref{AQE_CPL} and Proposition \ref{distrosem}:
\begin{proof}
Let $\mathfrak{G}$ be a lifted Bayesian network in the form of Proposition \ref{distrosem}. Then every $\chi_{R,i}$ is asymptotically equivalent to a quantifier-free first-order formula over the distribution defined by the probabilistic facts, which give the required quantifier-free representation. 
\end{proof} 
\subsection{Proof of Theorem \ref{functionalLBN}}

For the proof we will use the classical model-theoretic notion of an \emph{extension axiom}~\cite{KolaitisV92}.
The extensions axioms formalise the notion that any possible configuration of relations within a structure do occur. 

\begin{definition}[\citeA{KolaitisV92}]

  A \emph{complete atomic diagram in $n$ variables for a signature $\sigma$} is a quantifier-free formula $\varphi(x_1, \dots x_n)$ such that for every atom $R(\vec{x})$ of $\sigma$ with $\vec{x} \subset \{ x_1, \dots x_n\}$ either $\varphi(x_1, \dots x_n) \models R(\vec{x})$ or $\varphi(x_1, \dots x_n) \models \neg R(\vec{x})$.

An \emph{r+1-extension-axiom} for a language $L$ is a sentence 
\[
\forall_{x_{1},\ldots,x_{r}}\left(\underset{1\leq i<j\leq r}{\bigwedge}x_{i}\neq x_{j}\rightarrow\exists_{x_{r+1}}\varphi_{\Phi}(x_1, \dots, x_{r+1})\right)
\]
where $r\in\mathbb{N}$,
\[ \varphi_\Phi :=
\left(\underset{1\leq i\leq r}{\bigwedge}x_{i}\neq x_{r+1}\wedge\underset{\varphi\in\Phi}{\bigwedge}\varphi\wedge\underset{\varphi\in\Delta_{r+1}\backslash\Phi}{\bigwedge}\neg\varphi\right)
\]
and $\Phi$ is a subset of
\[
\Delta_{r+1} := \left\{ R(\vec{x})\mid R\in\mathcal{R},\textrm{ \ensuremath{x_{r+1} \in \vec{x}} a tuple from \ensuremath{\{x_{1},\ldots,x_{r+1}\}}} \right\} .
\]
\end{definition}

\begin{fact}[\protect{\cite[Theorem 3.13]{KolaitisV92}}]
    As a first-order theory, the set of extension axioms has quantifier elimination.
\end{fact}
In the following statements and proofs, we say that a sentence or other statement about finite structures $\varphi$ is true 
\emph{$(\mathbb{P}_{\vec{n}})$-almost-everywhere} for a family of distributions $(\mathbb{P}_{\vec{n}})$ if for any sequence $D_k$ of domains which tend uniformly to infinity, $\underset{n \rightarrow \infty}{\lim}\mathbb{P}_{\vec{n}}(\varphi) = 1$ 
\begin{lemma}\label{FirstLemma} Let the family of distributions $(\mathbb{P}_{\vec{n}})$ of $\{R_1, \dots, R_l\}$ be asymptotically equivalent to the reduct of a distribution on $\{R_1, \dots, R_l, P_1, \dots, P_m\}$ where $P_1, \dots, P_m$ are independently distributed with probabilities $p_1, \dots, p_m \in (0,1)$ and $R_1, \dots, R_l$ are defined to be Boolean combinations of $P_1, \dots, P_m$ that are neither contradictory nor tautologies.
Furthermore, assume that the following hold for $(\mathbb{P}_{\vec{n}})$:
\begin{enumerate}
\item For all $0 \leq q < l$ and all different tuples $\vec{a}$ and $\vec{b}$, $R_{q+1}(\vec{a})$ is conditionally $\mathbb{P}_{\vec{n}}$-independent of $R_{q+1}(\vec{b})$ given the interpretation of $R_1, \dots R_q$.

\item The probabilities $\mathbb{P}_{\vec{n}}(R_{i}(\vec{a}))$ do not depend on $\vec{a}$. 
\end{enumerate}

Let $\chi$ be an arbitrary extension axiom. Then  $\chi$ holds $\mathbb{P}_{\vec{n}}$-almost-everywhere.
\end{lemma}
\paragraph{Proof of Lemma \ref{FirstLemma}}
For the sake of simplifying notation, we will only consider the single-sorted case. For the multi-sorted case, the calculation is analogous. 
By induction on $l$.

$l = 1$: 
Choose an arbitrary $r+1$ extension axiom $\chi$ corresponding to a set $\Phi$. Let $\vec{a}$ be arbitrarily chosen. Then by asymptotic equivalence, for sufficiently large $\vec{n}$, $\mathbb{P}_n(\varphi_{\Phi}(\vec{a},y))>\delta$ for a $\delta >0$.
By conditions 1 and 2, we can conclude that
\[
(\mathbb{P}_n)(\neg \chi) \leq n^r{(1-\delta)}^{n-r}
\]
which limits to $0$ as $n$ approaches $\infty$.

$l \rightarrow l+1$:
Choose an arbitrary $r+1$ extension axiom $\chi$ corresponding to a set $\Phi$. By the induction hypothesis, all extension axioms hold for $R_1, \dots, R_l$. Therefore, for any natural number $k$, there are $(\mathbb{P}_n)$-almost-everywhere more than $k$ witnesses of $\varphi_{\Phi \cap \{R_1, \dots, R_l\}}(\vec{a},y)$ for every $\vec{a}$, for arbitrary $k$. Just as in the base case above, we can now apply asymptotic equivalence and  conditions 1 and 2 to conclude that conditioned on there being at least $k$ witnesses $\varphi_{\Phi \cap \{R_1, \dots, R_l\}}(\vec{a},y)$,
\[
(\mathbb{P}_n)(\neg \chi) \leq k^r{(1-\delta)}^{k-r}
\]
which limits to $0$ as $k$ approaches $\infty$.  $\square$

The next lemma is a version of the strong law of large numbers adapted to our setting. 

\begin{lemma}\label{SecondLemma} Let $\phi$ be a complete atomic diagram in $r$ variables. For every $p \in (0,1)$ consider the subset $\Omega_{\phi,p}$ of $\Omega_\infty$ defined by the event that the relative frequencies of $\phi$ in the initial segments of $\omega \in \Omega_\infty$ of lengths $\vec{n}n$ converge to $p$. Let $\mathbb{P}_\infty$ be the asymptotic limit of an independent distribution.  Then there is a  $p_\phi \in (0,1)$ such that  $\mathbb{P}_\infty(\Omega_{\phi,p}) = 1$. This also holds when conditioning on a complete atomic diagram in $m < r$ variables. 
\end{lemma}

\paragraph{Proof of Lemma \ref{SecondLemma}}
Induction on $r$.

In the case $r = 1$,  $\phi(x)$ is asymptotically equivalent to a sequence of independent random variables, and the strong law of large numbers gives the result.

So assume true for $r$. Then $\phi(\vec{x},y)$ is given by $\phi_r(\vec{x}) \wedge \phi'(\vec{x},y)$, where $\phi'(\vec{x},y)$ is a sequence of independent random variables and $\phi_r(\vec{x})$ is a complete atomic diagram in $r$ variables.
We can conclude with the strong law of large numbers again. $\square$

\paragraph{Proof of Theorem \ref{functionalLBN}}
We now proceed with the proof of Theorem \ref{functionalLBN} itself.
The overall strategy is to proceed inductively, using Lemma \ref{FirstLemma} to show that the extension axioms are valid almost everywhere,
and then to apply quantifier elimination of the resulting first-order theory.
We can then conclude with Lemma \ref{SecondLemma} to show that the relative frequencies of quantifier-free formulas are convergent,
and continue the induction. 

Recall that the \emph{height} of a directed graph is the length of its longest directed path. 
We perform a parallel induction by height on the following statements:

The family of distributions $(\mathbb{P}_{\vec{n},T})$ is asymptotically equivalent to a quantifier-free LBN, in which all aggregation formulas $\varphi_{R,i}$ for which neither $ \forall_{\vec{x}}(\varphi_{R,i}(\vec{x}) \rightarrow R(\vec{x}))$ nor $\forall_{\vec{x}} \neg (\varphi_{R,i}(\vec{x}) \wedge R(\vec{x}))$ are true $(\mathbb{P}_{\vec{n},T})$-almost-everywhere, are annotated with a probability in $(0,1)$. 

Every extension axiom for the language with those relations that have asymptotic probability $p \in (0,1)$  is valid $(\mathbb{P}_{\vec{n},T})$-almost-everywhere 

Let $T_h$ be the fragment of the FLBN of height not exceeding $h$. 

Base step:
An FLBN of height 0 is an independent distribution, showing the first statement. The extension axioms are well-known to hold almost everywhere with respect to such a family of distributions~\cite{KolaitisV92,KeislerL09}.

Induction step:
Since the extension axioms form a complete theory with quantifier elimination, every $\chi_{R,i}$ is $(\mathbb{P}_{\vec{n},T_h})$-almost-everywhere equivalent to a quantifier-free $\chi'_{R,i}$. Since we can replace those relations that are almost everywhere true for all or no elements by $\top$ and $\bot$ respectively, we can assume no such relations to occur in $\chi'_{R,i}$.

Let $\varphi_1, \dots, \varphi_m$ list the complete atomic diagrams in the variables $x_1, \dots x_r$ for the signature of $T_h$, where $r$ is the arity of $R$. Then by Lemma 2 above, for every $\varphi_j$, the relative frequency of $\chi'_{R,i}$ given $\varphi_j$ is almost surely convergent to  a number $p_{\chi'_{R,i},\varphi_j}$, which lies in $(0,1)$ if and only if $\chi'_{R,i}$ is neither logically implied by nor contradictory to $\varphi_j$.

Therefore, the family of distributions $(\mathbb{P}_{\vec{n},T})$ is asymptotically equivalent to a quantifier-free LBN, which is obtained by listing $\varphi_1, \dots, \varphi_m$ and annotating them with the probability $f_R(p_{\chi'_{R,1},\varphi_j}, \dots, p_{\chi'_{R,n},\varphi_j})$.

It remains to show that the additional condition on the annotated probabilities is satisfied.

So assume that $f_R(p_{\chi'_{R,1},\varphi_j}, \dots, p_{\chi'_{R,n},\varphi_j}) \in \{0,1\}$.

By the assumption on $f_R$, this implies $p_{\chi'_{R,i},\varphi_j} \in \{0,1\}$ for every $i$.
Since each $\chi_{R,i}$ is $(\mathbb{P}_{\vec{n},T})$-almost-everywhere equivalent to $\chi'_{R,i}$, the relative frequency of $\chi_{R,i}$ given $\varphi_j$ is exactly $p_{\chi'_{R,i},\varphi_j}$ almost everywhere. So almost everywhere the conditional probability of $R(\vec{x})$ given $\varphi_j(\vec{x})$ is 0 or 1. In other words, $ \forall_{\vec{x}}(\varphi_j(\vec{x}) \rightarrow R(\vec{x}))$ or $\forall_{\vec{x}} \neg (\varphi_j(\vec{x}) \wedge R(\vec{x}))$ are true $(\mathbb{P}_{\vec{n},T})$-almost-everywhere.

Finally we show that the extension axioms are valid almost surely. We will verify the assumptions of Lemma \ref{FirstLemma}. By asymptotic equivalence to a quantifier-free LBN and Proposition 1, we can find an independent distribution as required by Lemma \ref{FirstLemma}. The two additional assumptions are clearly satisfied for any FLBN by the Markov condition for the underlying Bayesian network. $\square$

\subsection{Examples}\label{subsec:examples}
We illustrate the analysis of the last subsection with a sequence of simple examples that serve to highlight the main aspects, and conclude with two counterexamples that demonstrate the necessity of two of the restrictions imposed on our setting.
Consider the situation of Example \ref{example:RLR}: The signature $\sigma$ has two unary relation symbols  $Q$ and $R$, and the underlying DAG $G$ is $Q \longrightarrow R$. We  model a relationship between $R(x)$ and those $y \in D$ that satisfy $Q(y)$.
In  Example \ref{example:RLR}, we have seen an RLR approach to this problem. Here, the asymptotic behaviour is well-known: as domain size increases, the expected number of $a \in D$ that satisfy $Q(a)$ also does. By the law of large numbers, this increase is almost surely linear with domain size. Therefore the probability of $R(y)$ will limit to 0 if and only if $w < 0$, and limit to 1 if and only if $w > 0$.
A similar analysis holds if we consider a noisy-or combination, or the model of a probabilistic logic program.  
\begin{example}\label{example:noisy-or}
Assume that for every $y \in D$ for which $Q(y)$ holds there is an independent chance $p_1 \in [0,1]$ of $R(x)$, and that for every $y \in D$ for which $Q(y)$ does not hold there is an independent chance $p_1 \in [0,1]$ of $R(x)$. Further assume that $\mu(Q) \in (0,1)$. Then as domain size increases, there will almost surely be at least one $y \in D$ that causes $R(x)$, leading to $R(x)$ being true almost surely --- unless two of $p_1$, $p_2$ and $\mu(Q)$ are 0. So outside of these boundary cases, $p_1$, $p_2$ and $\mu(Q)$  have no influence on the asymptotic probability of $R(x)$ (which is always 1).
In the representation of the distribution semantics, there would be additional binary predicates $P_1(x,y)$ and $P_2(x,y)$ with $\mu(P_1) := p_1$ and  $\mu(P_2) := p_2$. the definition of $R(x)$ would be$\exists_y (Q(y) \wedge P_1(x,y)) \vee \exists_y (\neg Q(y) \wedge P_2(x,y))$.  
Asymptotically, when neither of the root probabilities are zero, both existential clauses in that definition will evaluate to ``true''. 
\end{example}

If we assume the dependency to be discrete with a known cut-off point $r$ in the relative frequency of $R$, we could consider using LBN-CPL to model it:

\begin{example}\label{example:LBN-CPL}
 Let $\mathfrak{G}$ be an LBN-CPL on $G$ with a probability $\mu(Q) \in [0,1]$ for $Q$ and two conditional probability formulas characterising $R$: $\chi_{R,1} := \left\Vert Q(y)\right\Vert_y \geq r$ and  $\chi_{R,2} :=\neg (\left\Vert Q(y)\right\Vert_y \geq r)$, where $r \in [0,1]$. We furthermore choose $\mu(R \mid \chi_{R,1}) \in [0,1]$ and $\mu(R \mid \chi_{R,i}) \in [0,1]$.  It turns out that the formulas $\chi_{R,i}$ are non-critical only if $r \neq \mu(Q)$. The asymptotic analysis here proceeds as follows: By the law of large numbers,  $\left\Vert Q(y)\right\Vert_y$ will be almost surely arbitrarily close to  $\mu(Q)$ as domain size increases, and therefore almost surely $ \chi_{R,1}$ will be true if and only $r <  \mu(Q)$. Thus the asymptotically equivalent quantifier-free Bayesian network will simply have the quantifier-free formula ``true'' as $\chi_{R}$ and $\mu(R)$ will be  $\mu(R \mid \chi_{R,1})$ if  $r <  \mu(Q)$ and  $\mu(R \mid \chi_{R,2})$ if  $r >  \mu(Q)$. Note that we cannot make any statement about the critical case $r = \mu(Q)$. 
Alternatively, consider the representation in the form of Proposition \ref{distrosem}, which is as follows: The signature $\sigma '$ has two additional unary predicates $P_1$ and $P_2$, with $\mu(P_i) := \mu(R \mid \chi_{R,i})$. The definition of $R(x)$ is then given by $R(x) := (\chi_{R,1}(x) \wedge  P_1(x)) \vee (\chi_{R,2}(x) \wedge  P_2(x))$.
Asymptotically, $\chi_{R,1}(x)$ and $\chi_{R,1}(x)$ behave just as in the original representation, so the asymptotic representation will be $R(x) :=   P_1(x)$ if  $r <  \mu(Q)$ and $R(x) :=   P_2(x)$ if  $r >  \mu(Q)$.
\end{example}
Now consider modelling such a dependency with an FLBN.\@ 
\begin{example}\label{example:FLBN}
  Let $\mathfrak{G}$ be an FLBN on $G$ with a formula  $Q(y)$ and a function $f_R : [0,1] \rightarrow [0,1]$. Assume further that $Q$ is annotated with a probability $\mu(Q) \in [0,1]$. Then by the law of large numbers, $\left\Vert Q(y)\right\Vert_y$ converges to  $\mu(Q)$ almost surely as domain size increases. Since $f$ is continuous, this implies that $f(\left\Vert Q(y)\right\Vert_y)$  converges to $f(\mu(Q))$ almost surely. So the asymptotically equivalent quantifier-free network will have ``true'' as its formula for $R$ and then $f(\mu(Q))$ as $\mu(R(\vec{x}) \mid \mathrm{true})$. 
\end{example}

We give a counterexample to show that convergence to a quantifier-free network can fail in the presence of non-continuous aggregation functions.

\begin{example}
  Consider a language with two sorts and two unary relation symbols $P$ and $R$, each with their own sort,
  and an underlying graph $P \longrightarrow R$.
  Let $P$ be annotated with the probability 0.5 and let $R$ be annotated with the aggregation function $f_R:[0,1] \rightarrow [0,1]$, where $f_R(a) = 0.7$ if $a \geq 0.5$ and $f_R(a) = 0.2$ otherwise. Note that this does not give rise to an FLBN by our definition because $f_R$ is not continuous.
  As the domain size increases, the probability that the relative frequency of $P$ among domain elements is at least 0.5 approaches 0.5.
  Therefore, asymptotically, the relative frequency of domain elements satisfying $R$ will either be close to 0.2 or close to 0.7, each with a probability of 0.5.
  This behaviour cannot be modelled by a quantifier-free lifted Bayesian network, as such a network could only specify a single probability number for the relation symbol $R$ (the only relation symbol mentioning its sort).
\end{example}

Another marked restriction of our framework is the omission of equality from our language.
Indeed, including equality also causes the failure of asymptotic convergence.
The counterexample relies on a famous result by \citeA{ShelahS88}, accessibly expounded  in Spencer's~\citeyear{Spencer01} monograph, that sparse random graphs with edge probabilities $(1/n)^{\alpha}$  with
rational coefficients $\alpha$ do not have limit probabilities for first-order logic.

\begin{example}\label{exam:ShelahSpencer}
  Consider a two-sorted domain with a binary relation symbol $E$, with both entries of the first sort, and a unary relation symbol $R$ with one entry of the second sort.
  Consider the underlying graph $P \longrightarrow R$.
  Define an edge relation $E$ with the formula $x=y$ and the aggregation function $f_E:[0,1] \rightarrow [0,1]$, where $f_E(a) = \sqrt(a)$.
  This corresponds to a random graph in which edges are drawn with probability $(1/n)^{(1/2)}$. 
  By the main result of \citeA{ShelahS88},
  there is a first-order $\{E\}$-sentence $\varphi$ such that the probability that $\varphi$ holds in the random graph with $n$ nodes and edge probability   $(1/n)^{(1/2)}$ diverges with increasing $n$. 
  Now consider the aggregation function $f_R$ to be 1 if $\varphi$ holds and 0 otherwise.
  In this case, $R(x)$ will hold for either all or none of the elements of the domain, with the likelihood of both scenarios diverging as $n$ increases.
  This clearly rules out asymptotic convergence to a quantifier-free lifted Bayesian network, since in those the proportion of elements satisfying $R$ would converge to the probability $R$ is annotated with.
\end{example}

\begin{remark}
  Note that both examples could easily have been made single-sorted;
  the two-sorted domain was chosen merely to bring out the idea more clearly by limiting the possible quanitifer-free networks that had to be ruled out as potential limits. 
\end{remark}

\subsection{Transfer Learning across Domain Sizes}

As projective families of distributions, quantifier-free lifted Bayesian networks have very desirable properties for learning across domain sizes. 
More precisely, for the family of distributions induced by any quantifier-free Bayesian network and any structure $\mathfrak{X}$ with $m < n$ elements,
$\mathbb{P}_n (\mathfrak{X}) = \mathbb{P}_m(\mathfrak{X})$.

Consider a parametric family of distributions $G_{\theta}$ which
are asymptotically equivalent to a parametric projective family of
distributions $G_{\theta}'$. Consider the problem of learning the
parameters from interpretations on a structure $\mathfrak{X}$ of
large domain size $n$. Then we could proceed as follows: Sample
substructures of domain size $m<n$, where $m$ is larger than the highest arity in $\sigma$ and the arity of the queries we are typically interested in. 
Find the parameters of $G_{\theta}'$ that maximise the sum of the log-likelihoods of the samples of size $m$. Now
consider $G_{\theta}$. By the asymptotic convergence results, if $n$ is sufficiently large, these
parameters maximise the likelihood of obtaining the substructures of size $m$ sampled from $\mathfrak{X}$ using $G_{\theta}$,
including realisations of the queries we might be interested in.

For this approach to work, the parametric families $G_{\theta}'$ must actually depend on the parameters, and do so in a regular way;  for the learning algorithms discussed
above that means that they should be differentiable in the parameters.

Let us evaluate these criteria for the asymptotic approximations in  Examples~\ref{example:RLR}-\ref{example:FLBN}:
In the case of Example \ref{example:RLR}, $G_{\theta}'$ does not depend on $w$ beyond its sign.
Furthermore, there is a discontinuity at $w = 0$.
The asymptotic probability of the noisy-or model in Example \ref{example:noisy-or} does not depend at all on $p_i$ as long as $p_i \neq 0$.

In the  functional lifted Bayesian network model of Example \ref{example:FLBN}, the
parametric family is defined by $f_{\theta}(\mu_{Q})$. If $f$ is
linear or logistic, for instance, then $f_{\theta}(\mu_{Q})$ is
differentiable in the parameters for any fixed $\mu_{Q}\in(0,1)$.
Note, that while $f(\mu_{Q})$ will vary with every parameter individually,
it will take its maximum-likelihood value (which happens to coincide
with the true frequency of $R(x)$ in the domain) on an infinite subspace
of tuples of parameters. This is not unique to the projective approximation, however,
but is a well-known phenomenon when learning the parameters for a
relational logistic regression from a single interpretation~\cite{KazemiBKNP14,PooleBKKN14}.
This can be overcome by learning from several domains, where $Q$ has different frequencies.  

In the conditional probability modelling of Example \ref{example:LBN-CPL}, we could
start by estimating $\mu_{Q}$ from data on the values of $Q$. If
$\mu_{Q}\neq r$, we can then proceed to learn one of the parameters
using the asymptotic limit as outlined above. In this parameter, the
dependence is clearly linear and therefore satisfies all of the conditions.
However, the other parameter does not occur in the projective limit
and therefore cannot be estimated in this way.
To estimate that parameter also, we would also need more training domains, including those where the relative frequency of  $Q$ is above the threshold

In this way, Type III formalisms allow us to leverage the power of projective families of distributions for transfer learning while retaining much more expressive modelling capabilities. While adding either functional or discrete dependencies on the Type I probabilities present in a domain allow us to express rich connections between different domain elements, quantifier-free lifted Bayesian networks themselves do not allow any dependence on the global structure of the model. This is also quite typical of projective families of distributions that can be expressed in statistical relational AI;\@ for instance, any projective LBN-CPL is expressible by a quantifier-free one (as two asymptotically equivalent projective families of distributions are completely equivalent), and every projective probabilistic logic program is determinate~\cite{Weitkaemper21a}.

We would like to end this section by remarking how our asymptotic results complement those of  \citeA{Jaeger98}. There, Jaeger shows that the probability distributions of relational Bayesian networks with \emph{exponentially convergent} combination functions lead to asymptotically convergent probability distributions. However, the exponentially convergent combination are essentially those that given a certain type of input sequence increasing in length converge to a fixed value, regardless of the precise sequence received. The classical combination functions `noisy-or' and `maximum' are paradigmatic for this behaviour. The central idea of our work here is precisely that the functions converge to a value that depends explicitly on the means of the sequences received, and therefore they are clearly distinguished in their behaviour from Jaeger's exponentially convergent combination functions.
This idea is made precise in the concept of uniform convergence. 
\subsection{Uniform Convergence}
We now show that not only do the distributions induced by functional lifted Bayesian networks converge, they do so uniformly in the parameters. 
The proofs presented in this section require a basic knowledge of analysis in several variables as well as the Arzela-Ascoli Theorem from functional analysis, as they can be found in any comprehensive basic reference text. 
We will be relying on the textbook by \citeA{Deitmar21}. 
Consider for a given target relation $R\in \sigma$ a finite tuple of formulas as before, and then a function 
\[f:[0,1]^{n_R} \times K_1 \times \dots \times K_m \rightarrow \mathbb{R}\]
which takes as arguments not merely the relative frequencies $x_1,\dots,x_{n_R}$ of the associated formulas but also $m$ weight parameters $w_1,\dots,w_m$ taking values in compact sets $K_i \subset \mathbb{R}$. Furthermore we assume that the functions are not merely continuous, but continuously differentiable.
This applies for instance for all the examples of generalised linear models mentioned in \ref{ExpressFLBN}

Consider the query $R(\vec{a})$, possibly conditioned on some evidence regarding the parents of $R$, on a domain of size $n$. 
The answer to the query can then be computed as $f_n(w_1,\dots,w_m)$, the integral of
\[f(x_1,\dots,x_{n_R},w_1,\dots,w_m)\]
over the discrete probability distribution on $\Omega_n$ defined by the lower levels of the FLBN and the evidence. 

Theorem \ref{functionalLBN} shows that for any given set of weights, as $n$ increases, the resulting sequence of query probabilities converges to the value computed by the asymptotically equivalent projective network. 

In this section, we show that this convergence is uniform in the weight parameters and highlight some consequences of this fact. 
To simplify the argument, we only consider varying the functions associated with a single relation symbol.  

\begin{definition}
    Let $X$ be a set and $(f_n)$ a sequence of functions from $X$ to $\mathbb{R}$. 
    Then $(f_n)$ \emph{converges uniformly} to a function $f:X\rightarrow \mathbb{R}$ if for any $\varepsilon > 0$ there is an $n_0\in \mathbb{N}$ such that for all $x \in X$ and for all $n > n_0$, $\lvert f_n(x) - f(x) \rvert < \varepsilon$.  
\end{definition}

Note that $n_0$ is chosen uniformly across $x \in X$, in contrast to pointwise convergence in which a separate $n_0$ can be chosen for each $x$.  

\begin{theorem}\label{thm:uniform}
    
    Let $\sigma$ be a relational signature, let $R\in \sigma$ be an $n$-ary predicate and let $\left(\mathfrak{G}_{w_1,\dots, w_m}\right)$ be a family of FLBNs on $\sigma$ which agree on the underlying graph and on all functions associated with any relation symbol in $\sigma \setminus \{R\}$, where $w_1,\dots,w_m$ take values in compact subsets $K_1,\dots,K_m$ of $\mathbb{R}$ respectively.  
    Let $f:[0,1]^n \times K_1 \times \dots \times K_m \rightarrow [0,1]$ be continuously differentiable. 
    In $\mathfrak{G}_{w_1,\dots,w_m}$, let $f_R(x_1,\dots,x_n) =  f(x_1,\dots,x_n,w_1,\dots,w_m)$.
    
    Let $D_0$ be a choice of domains for $\sigma$ and let $a_1,\dots,a_n, b_1,\dots,b_l \in D_0$.
    Let $\varphi$ be a quantifier-free formula with $l$ variables using only relations among the ancestors of $R$ in the graph underlying $\mathfrak{G}$. 
    Then for any unbounded and monotone sequence of domains $D_k$ extending $D_0$, the sequence of probabilities of the query $R(a_1, \dots, a_n)$ conditioned on $\varphi(b_1,\dots,b_l)$ induced by $\mathfrak{G}_{w_1,\dots, w_m}$ converges uniformly in the parameters $w_1,\dots, w_m$.

\end{theorem}
\begin{proof}
    
Our key ingredient is the Arzela-Ascoli Theorem~\cite[Theorem 8.6.4]{Deitmar21}, which we present in a version slightly specialised to our setting.

It relies on the notion of equicontinuity:
\begin{definition}
    Let $X$ be a metric space and $(f_n:X \rightarrow \mathbb{R})$ a sequence of functions. 
    Then $(f_n)$ is called \emph{equicontinuous} at a point $x \in X$ if 
    for every $\varepsilon > 0$ there is a $\delta > 0$
    such that for any $y \in X$, whenever $d(x, y) < \delta$,  $\lvert f_n(x) - f_n(y) \rvert < \varepsilon$ uniformly for all $n \in \mathbb{N}$. 
\end{definition}
\begin{fact}
Let $X$ be a compact subset of $\mathbb{R}^m$. 
Let $(f_n)$ be a sequence of continuous functions from  $X$ to $[0,1]$. If
this sequence is equicontinuous at every point, then it possesses a uniformly convergent subsequence. 
\end{fact}

To see that our sequence of functions satisfy the requirement of the Arzela-Ascoli Theorem, consider the $m$-dimensional gradient
$\vec{f}'(x_1,\dots,x_{n_R},w_1,\dots,w_m)$  
of $f$ in the coordinates $w_1,\dots,w_m$.
By assumption, the gradient function is continuous.
Consider the function
\[g:B^m \times [0,1]^{n_R} \times K_1 \times \dots \times K_m \rightarrow \mathbb{R}\]
which for every $m$-dimensional unit vector $v$ computes the absolute value of the directional gradient of $f$ in the direction $v$ interpreted as coordinates $w_1,\dots,w_m$.
Since the gradient of $f$ is continuous, so is $g$. 
As a continuous function on a compact domain, $g$ is bounded by a value $g_{\mathrm{max}} \in \mathbb{R}$.
Now, for any $\varepsilon > 0$, choose $\delta < \varepsilon \div g_{\mathrm{max}}$. 
Let $\vec{w_0}$ be any tuple of weights and let $\vec{w}$ be $\delta$-close to $\vec{w_0}$.
Then by the mean value theorem, for any $\vec{x}$ the distance between $f(\vec{x},\vec{w})$ and $f(\vec{x},\vec{w_0})$ is at most $\delta * g_{\mathrm{max}} < \varepsilon$.
Now consider the sequence of functions $f_n$, and let $\vec{w_0}$ and $\vec{w}$ be as above. 
Then $f_n(\vec{w_0}) - f_n(\vec{w})$ is at most the integral of the differences between $f(\vec{x},\vec{w_0})$ and $f(\vec{x},\vec{w})$, and since all of these are bounded by $\varepsilon$, so is their integral over a probability distribution. 

Therefore, the Arzela-Ascoli Theorem implies that for any atomic query $R(\vec{a})$ and any evidence whose distribution is computed at lower levels of the FLBN, the sequence $p_n$ of conditional probabilities of $R(\vec{a})$ given the evidence and a domain of size $n$ containing all constants mentioned by the query and the evidence is uniformly convergent in the weight parameters.   
\end{proof}

\pagebreak
We can deduce two corollaries from this: 
\begin{corollary}
Under the assumptions of Theorem \ref{thm:uniform}, the asymptotic probability of an atomic query is a continuous function of the weight parameters. In particular, if the parameter space is connected and there are choices of parameters $\vec{v}$ and $\vec{w}$ for which the asymptotic probabilities are $p$ and $q$ respectively, then for every $p'$ that lies between $p$ and $q$ there is a choice of parameters $\vec{w'}$ such that the asymptotic probability for $\vec{w'}$ is precisely $p'$. 
\end{corollary}
\begin{proof}
    The uniform limit of a sequence of continuous functions is itself a continuous function~\cite[Prop. 7.1.3]{Deitmar21}. The second part of the corollary follows from the first part and the intermediate value theorem for continuous functions. 
\end{proof}

This is in sharp contrast to the situation for typical statistical relational formalisms such as probabilistic logic programming or ordinary relational logistic regression, 
where the asymptotic probability of the atomic query jumps from 0 to 1 depending on whether the parameter is zero or not (PLP) or negative or positive (RLR). 
We can also conclude from the converse statement that in these formalisms,  convergence is not uniform in the parameter space. 

Note that sometimes, weight parameters can take values on all of $\mathbb{R}$, which is itself not compact.
However, continuity is a local property, so the argument made here suffices to prove continuity on any compact set and thus in a compact neighbourhood of every point. 
Thus, the limit is continuous at every value of the weight parameters. 

\begin{corollary}\label{cor:uniform}
    Under the conditions of Theorem \ref{thm:uniform}, for every $N \in \mathbb{N}$ and every  $\varepsilon > 0$ there are natural numbers $\vec{n} \in \mathbb{N}$, one for each sort, such that for every choice of domains $D$ larger than $\vec{n}$, and for all possible quantifier-free formulas $\varphi$ with at most $N$ free variables, and for all $a_1, \dots, a_k$ from $D$ matching the sorts of $R$ and all $b_1, \dots, b_N$ matching the sorts of $\varphi$,  
    the difference between the probability of $R(a_1, \dots, a_k)$ given $\varphi(b_1, \dots, b_N)$  induced by $\mathfrak{G}_{\vec{w}}$ and that induced by its asymptotic limit is less than $\varepsilon$.
\end{corollary} 
\begin{proof}
    For every individual formula $\varphi$, this follows immediately from the definition of uniform continuity.  
    Since all involved probability distribution are invariant under renaming of the constants, it suffices to consider all possible quantifier-free with at most $N$ variables, of which there are only finitely many up to logical equivalence, and all possible sort-appropriate tuples from  disjoint sets of elements $\{a_1, \dots, a_k\}$ and $\{b_1, \dots, b_N\}$, including those which overlap. 
    Of these there are again finitely many. 
    Thus for every sort we can choose the maximum of the domain sizes corresponding to each choice of $\varphi$ and pair of tuples. 
\end{proof}

Corollary \ref{cor:uniform} ensures that the convergence results still hold in the situation of initially unknown parameters, as is the case in parameter learning, and that guarantees of closeness of fit can be derived across a range of formulas and parameters of interest.

\section{Further Work}

By defining the semantics of FLBN via Bayesian networks, we can use all inference methods developed for Bayesian networks in the same way as they have been exploited for Relational Bayesian Networks. This includes importance sampling \citeA{Jaeger06}, knowledge compilation techniques \citeA{ChaviraDJ06} and first-order variable elimination~\cite{MilchZKHK08}. These techniques allow for conditioning on given data and also let us prescribe the interpretation of root predicates if desired.

We  now briefly consider the methods and algorithms that are available for learning
LBN-CPL and FLBN. 
Let us consider LBN-CPL first. Here, parameter learning from interpretations amounts to learning the probabilities $\mu(R \ \mid \ \chi_{R, i})$ given one (or more) $\sigma$-structures. This case is closely aligned to parameter learning in other directed settings. In particular, in the case of a completely specified training model, the probabilities can be set to equal the frequencies encountered in the data. This can actually be expressed in terms of CPL itself:
Set  $\mu(R \ \mid \ \chi_{R, i})$ to be equal to the relative frequency $\left\Vert R(\vec{x}) \mid \chi_{R, i}(\vec{x}) \right\Vert _{\vec{x}}$.
In the case of missing data, we can adapt the expectation-maximisation algorithm for learning parameters in ordinary Bayesian networks to CPL.\@ Since we can view the LBN-CPL as a deterministic model over probabilistic facts (Proposition \ref{distrosem}), this can be done by tying parameters known to be the same from the relational structure of the model and the formulas $\chi_{R, i}$ in just the same way as in probabilistic inductive logic programming (outlined for instance by \citeA[Section 8.5]{deRaedt08}). 
This native support for parameter learning makes it particularly attractive for modelling tasks such as the ones outlined above. In the context of infectious disease dynamics for instance, the parameters have epidemiological meanings and their estimation from past data is an important part of informed decision-making. 
While structure learning could in principle also be accomplished along the lines of the SLIPCOVER beam search algorithm~\cite{BellodiR15}, we anticipate that the significantly larger theory space resulting from the added probability quantifiers would make that difficult. Indeed, as far as we are aware even the problem of plain inductive logic programming in CPL without a Type III extension is yet to be addressed.

When posing the problem of parameter learning for FLBN, we first have to clarify which parameters we are learning in the first place. So rather than just considering functions $f: {[0,1]}^n \rightarrow [0,1]$ we consider a parametrised family of such functions, $f : K \times {[0,1]}^n \rightarrow [0,1]$, where $K$ is a connected subset of $\mathbb{R}^m$ from which the parameters are taken. For instance, in the cases of linear and logistic regression, Equations \ref{linear} and \ref{logistic} define functions taking  $m = n + 1$ parameters, $w_1, \ldots , w_n$ and $c$. In the logistic case, $K = \mathbb{R}^m$, while in the linear case, the parameter space is constrained by the function mapping to $[0,1]$. 
\citeA{Jaeger07} presents a general approach for learning the parameters of aggregation functions using gradient descent whenever the functions are differentiable in the parameters. Clearly both the linear and the logistic regression examples are differentiable in the parameters, and we believe this to be characteristic of a good choice of parameters. 
Functional gradient boosting has been successfully applied to the structure learning of relational logistic regression by \citeA{RamananKKFKPKN18}, and it seems very promising to evaluate this approach with other classes of regression functions expressible by FLBN.\@ 
We believe structure learning in FLBN to be a promising avenue for further research. Firstly, there is a large bank of work on regression learning in the statistical literature on which relational versions could conceivably be based, and secondly, the scale of the task can be reduced systematically by partly specifying families of functions  (recovering e.\ g.\ structure learning in relational logistic regression or linear regression as special cases).

The expressivity of FLBN is limited in several ways; some of the restrictions, such as continuity of the aggregation functions and the lack of equality in the first-order language, have been shown to be necessary in Subsection \ref{subsec:examples}.
On the other hand, many practical applications have a recursive structure, where the likelihood of a given relation holding for one tuple depends on whether this relation holds for other tuples in the domain.
In the example of infectious disease modelling, for instance, it is natural for infectedness of one individual directly to depend on the infectedness of their contacts.
Probabilistic logic programming allows for a direct statement of such recursive relationships, while undirected formalisms such as Markov logic networks can easily express such relationships as correlations.
However, neither Jaeger's~\citeyear{Jaeger98} asymptotic analysis for some Relational Bayesian networks nor Koponen's~\citeyear{Koponen20} analysis for LBN-CPL allow for any cyclical relationships. 
Incorporating such recursive relationships into FLBN and understanding their asymptotic behaviour would therefore be of high relevance, but require novel theoretical ideas.

\section{Conclusion}
 In this contribution, we introduce and evaluate functional lifted Bayesian networks (FLBN) as a statistical relational Halpern Type III formalism for continuous dependencies, which is sufficiently expressive to express generalised linear models.  This introduces continuous dependencies on relative frequencies into statistical relational artificial intelligence, making Halpern Type III probabilities available to this field. This is supplemented by an in-depth comparison with lifted Bayesian networks with conditional probability logic (LBN-CPL), a recently introduced formalism for discrete dependencies of relative frequencies. 
 By supporting discrete and continuous dependencies on relative frequencies, the Type II approaches presented here can express the complex relationships that are required  to model application domains such as infectious disease dynamics.
 Furthermore, the transparent relationship of the semantics to Bayesian networks allows the application of well-developed learning and inference approaches.

FLBN also advance statistical relational learning from large interpretations by supporting learning from randomly sampled subdomains. 
This is underpinned by a rigorous analysis of their asymptotic behaviour.   
In particular, it is shown that unlike other, Type II formalisms currently in use, a large subclass of parametric families of FLBN display uniform convergence in the parameter space, which is key for the asymptotic considerations to have practical relevance.

\acks{This contribution was supported by LMUexcellent, funded by the Federal Ministry of Education and Research (BMBF) and the Free State of Bavaria under the Excellence Strategy of the Federal Government and the L\"{a}nder.
We would also like to thank Vera Koponen for several insightful conversations, and Fran\c{c}ois Bry and Kailin Sun for their very helpful comments on an earlier version of the manuscript. Finally, we would like to thank the anonymous reviewers for their thoroughness, which significantly improved the manuscript.}

\vskip 0.2in
\bibliography{Probabilisticlogicbib}
\bibliographystyle{theapa}

\end{document}